\documentclass[letterpaper]{article} % DO NOT CHANGE THIS
\usepackage{aaai2026} % [submission] % DO NOT CHANGE THIS
\usepackage{times}  % DO NOT CHANGE THIS
\usepackage{helvet}  % DO NOT CHANGE THIS
\usepackage{courier}  % DO ss CHANGE THIS
\usepackage[hyphens]{url}  % DO NOT CHANGE THIS
\usepackage{graphicx} % DO NOT CHANGE THIS
\usepackage{natbib}  % DO NOT CHANGE THIS AND DO NOT ADD ANY OPTIONS TO IT
\usepackage{caption} % DO NOT CHANGE THIS AND DO NOT ADD ANY OPTIONS TO IT
\usepackage{algorithm}
\usepackage{algorithmic}

\usepackage{booktabs}
\usepackage{multirow}
\usepackage{amsmath}
\usepackage{amssymb}
\usepackage{amsthm}
\usepackage{cleveref}

\usepackage{newfloat}
\usepackage{listings}
\DeclareCaptionStyle{ruled}{labelfont=normalfont,labelsep=colon,strut=off} % DO NOT CHANGE THIS
\floatstyle{ruled}
\newfloat{listing}{tb}{lst}{}
\floatname{listing}{Listing}
\usepackage{array}    
\usepackage{amsmath}
\usepackage{amssymb}

\newtheorem{theorem}{Theorem}[section]

\newtheorem{lemma}[theorem]{Lemma}

\newtheorem{definition}[theorem]{Definition}

\newtheorem{remark}{Remark}

\title{M-Loss: Quantifying Model Merging Compatibility with Limited Unlabeled Data}
\author{
    Tiantong Wang\textsuperscript{\rm 1},
    Yiyang Duan\textsuperscript{\rm 1},
    Haoyu Chen\textsuperscript{\rm1, \rm 2},
    Tiantong Wu\textsuperscript{\rm 3, \rm *},
    Wei Yang Bryan Lim\textsuperscript{\rm 1,}\protect\thanks{Corresponding Authors.}
}
\affiliations{
    \textsuperscript{\rm 1}College of Computing and Data Science, Nanyang Technological University\\
    \textsuperscript{\rm 2}School of Computer and lnformation Technology, Beijing Jiaotong University\\
    \textsuperscript{\rm 3}Alibaba-NTU Global e-Sustainability CorpLab (ANGEL)\\
     tiantong001@e.ntu.edu.sg,
     tiantong.wu@ntu.edu.sg,
     bryan.limwy@ntu.edu.sg
}

\begin{document}

\maketitle

\begin{abstract}
Training of large-scale models is both computationally intensive and often constrained by the availability of labeled data. Model merging offers a compelling alternative by directly integrating the weights of multiple source models without requiring additional data or extensive training. However, conventional model merging techniques, such as parameter averaging, often suffer from the unintended combination of non-generalizable features, especially when source models exhibit significant weight disparities.
Comparatively, model ensembling generally provides more stable and superior performance that aggregates multiple models by averaging outputs. However, it incurs higher inference costs and increased storage requirements. While previous studies experimentally showed the similarities between model merging and ensembling, theoretical evidence and evaluation metrics remain lacking. To address this gap, we introduce Merging-ensembling loss (\textit{M-Loss}), a novel evaluation metric that quantifies the compatibility of merging source models using very limited unlabeled data. By measuring the discrepancy between parameter averaging and model ensembling at layer and node levels, \textit{M-Loss} facilitates more effective merging strategies. Specifically, M-Loss serves both as a quantitative criterion of the \textbf{theoretical feasibility} of model merging, and a guide for \textbf{parameter significance} in model pruning. Our theoretical analysis and empirical evaluations demonstrate that incorporating M-Loss into the merging process significantly improves the alignment between merged models and model ensembling, providing a scalable and efficient framework for accurate model consolidation.
\end{abstract}


\begin{links}
    \link{Code}{https://github.com/languangduan/mLoss}
\end{links}
%     \link{Extended version}{
%     arxiv xxxx.xxx
\section{Introduction}

In large‐scale deep learning, the reliance on large-scale labeled datasets and intensive computation limits the applicability of supervised methods~\cite{cottier2024rising,bang2024vtrain,manchanda2024generative,rakaraddi2024graph,wang2021deep,faraboschi2024reducing,guo2024more}. Model merging offers a promising alternative by directly fusing the weights of multiple pretrained or fine‐tuned models into a single network, thereby reducing the cost for data collection and training~\cite{jin2022dataless,li2023deep,tiesyadav}. It is particularly effective when labeled data is scarce or unavailable~\cite{yang2023adamerging,xu2024training,li2025training}. Additionally, merging multiple models often enhances performance, as the merged model benefits from the diverse perspectives and learned representations of its individual source models.

Despite its potential, model merging faces several challenges. Traditional methods averaging parameters without pre-processing can merge non-generalizable features~\cite{cade2015model,maleki2022generalizability}. In scenarios where source models exhibit significant weight disparities, direct parameter averaging can yield poor results~\cite{li2021fedbn}. \textbf{Notably, existing research lacks a theoretical tool to assess model merging compatibility without labeled data.}

A theoretical challenge in model merging is understanding why and when multiple non-linear models can be effectively merged through simple weight averaging. While prior research has explored the feasibility of model merging by drawing parallels with model ensembling—such as Mode Connectivity (MC) and Feature Connectivity (FC) studies~\cite{freeman2016topology,garipov2018loss,draxler2018essentially,zhou2023goingllfc} -- these studies primarily rely on empirical observations or theoretical analyses under specific conditions.

Model ensembling~\cite{dash2004model,hoeting1999bayesian,lv2023parameter,sagi2018ensemble,wen2020batchensemble} aggregates the outputs of multiple source models through linear averaging. It generally outperforms parameter averaging by providing more stable and accurate predictions with concrete theoretical evidence. However, ensembling necessitates storing multiple sets of model parameters and performing inference across all models, leading to increased computational and storage overhead.

To quantify the gap between model merging and model ensembling, we propose \textit{M-Loss}, an evaluation metric designed to quantify the mergeability of source models. M-Loss measures the discrepancy between the outputs of a weight-merged model and those of an ensembled model using only unlabeled data. M-Loss can be applied to multiple scenarios to assess the mergeability of source models theoretically and predict the performance of the merged model without a labeled test set. M-Loss can also work as a reference for developing new merging methods. To make the problem setting concrete and to clarify how M-Loss interfaces with practical merging algorithms, we provide an overview in Figure 1. It summarizes (i) how M-Loss quantifies the discrepancy between parameter averaging and ensembling using only unlabeled data, and (ii) how the resulting scores guide row-wise pruning schedules that plug into standard merging backends. In summary, our contributions are:

\begin{figure*}[t]
  \centering
  % 优先使用 PDF 矢量图；若为 PNG，请保证 300–600 dpi
  % 将文件名替换为你的导出文件：mloss.pdf / mloss.png
  \includegraphics[width=0.7\textwidth]{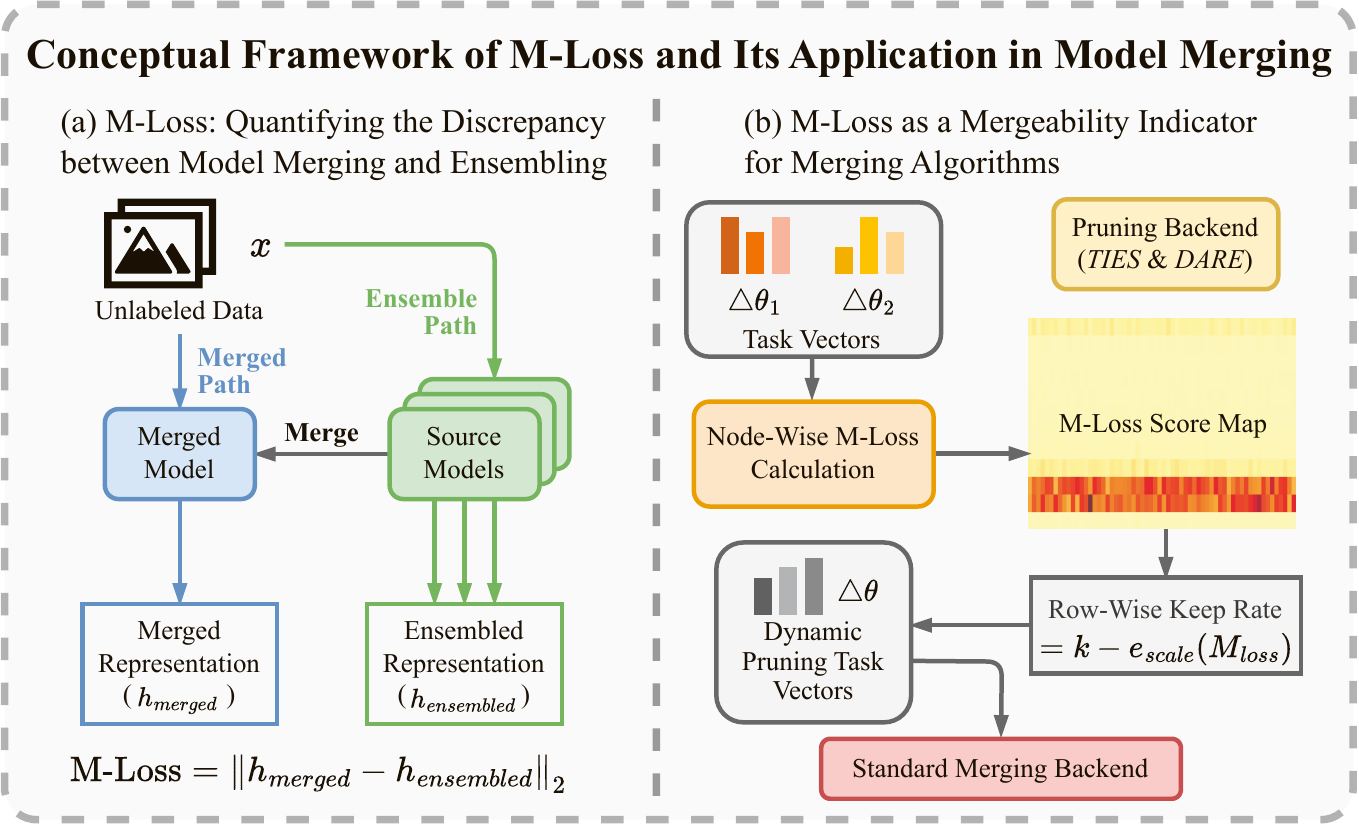}
  % \vspace{-6pt}
  \caption{
    Conceptual overview of M-Loss and its use in M-TIES.
    (a) M-Loss measures the discrepancy between parameter-averaged and ensembled representations on unlabeled data, producing layer-/node-wise scores.
    (b) The node-wise M-Loss score map drives dynamic row-wise keep rates, which integrate with standard merging backends (e.g., TIES Top-K or DARE) to improve mergeability and efficiency.
  }
  \label{fig:overview}
  % \vspace{-8pt}
\end{figure*}

\begin{enumerate}
    \item \textbf{M-Loss provides theoretical justification for model merging}: We address a fundamental question in model merging research: Why does a merged model, obtained through simple parameter averaging, closely approximate the ensembling of source models, despite non-linearity in model architecture? We propose a new evaluation metric called M-Loss, and derive the expectation of M-Loss under common activation functions, namely ReLU~\cite{nair2010rectifiedrelu}, GELU~\cite{hendrycks2016gaussiangelu}, and Leaky ReLU~\cite{maas2013rectifierleaky}. Our analysis provides theoretical justification for why fine-tuned models originated from a shared pretrained backbone can be effectively merged.
    
    \item \textbf{M-Loss as a reference for hyperparameter selection}: We show that M-Loss can act as a reliable indicator of parameter significance and inter-model conflicts, guiding the hyperparameter selection in model merging techniques such as TIES. By prioritizing parameters with lower M-Loss for retention, we propose the M-TIES method that employs a dynamical parameter pruning schedule for the merging process. 
\end{enumerate}

% This work establishes M-Loss as a foundational tool for advancing model merging techniques. By bridging the theoretical gap between parameter averaging and model ensemble, M-Loss not only provides rigorous justification for existing merging practices but also unlocks new opportunities for data-efficient and cross-domain model consolidation. Our empirical validation demonstrates its utility as a guide for optimizing merging algorithms. These insights pave the way for developing adaptive merging frameworks that dynamically balance performance and efficiency requirements for deploying scalable deep learning systems in real-world applications. The proposed methodology opens avenues for future research in automated model fusion and resource-constrained deep learning.

\section{Preliminaries and Related Works} 
\label{relatedworks}
\subsection{Model Merging}  
Model merging aims to consolidate multiple trained models into a single unified model by directly combining their parameters. Let $\{\theta^{(1)}, \theta^{(2)}, \dots, \theta^{(n)}\}$ denote the parameters of $n$ source models. A common approach of model merging is \textbf{simple averaging}~\cite{dash2004model,hoeting1999bayesian,wortsman2022modelsoups}, where the merged model $\theta_{\text{merged}}$ is computed as:  $\theta_{\text{merged}} = \frac{1}{n} \sum_{i=1}^n \theta^{(i)}$,
which implicitly assumes linear mode connectivity between models~\cite{garipov2018loss}. However, this simplistic approach faces great challenges under divergent optimization trajectories or when merging models fine-tuned for different tasks~\cite{li2019convergence}. To address this, recent methods like \textbf{Task Arithmetic}~\cite{ilharco2022taskarith} merge fine-tuned models by treating their parameter updates as additive task vectors relative to a shared pretrained initialization $\theta_{\text{pretrained}}$ as: $
\theta_{\text{merged}} = \theta_{\text{pretrained}} + \sum_{i=1}^n \frac{\lambda}{n} \left(\theta^{\left(i\right)} - \theta_{\text{pretrained}}\right),
$ with tunable hyperparameter $\lambda$ scaling the task vectors for better task adaptation. \textbf{AdaMerging}~\cite{yang2023adamerging} with 
$\theta_{\text{merged}} = \sum_{i=1}^n \alpha_i \theta^{(i)}$ learns the optimal merging weights $\{\alpha_i\}_{i=1}^n$ by output entropy. However, multiple rounds of merging are required to obtain the optimal weights, which increases the computational cost.
% Some pre-merging methods~\cite{tiesyadav,yu2024languagedare} further refine merging by pre-processing the task vectors before merging: 
% \[
% \theta_{\text{merged}} = \theta_{\text{pretrained}} + \sum_{i=1}^n \alpha_i \cdot \text{PreProcessing}\left(\theta^{(i)} - \theta_{\text{pretrained}}\right),
% \]  
% where $\text{PreProcessing}(\cdot)$ function aims to eliminate the parameter update conflicts between task vectors, such as different parameter updating directions, usually by pruning some unimportant or harmful parameters. These formulations highlight the tension between preserving useful features and suppressing interference during merging.

\subsection{Model Ensembling}  
Model ensembling~\cite{dash2004model,hoeting1999bayesian,lv2023parameter,sagi2018ensemble,wen2020batchensemble} aggregates predictions from multiple models rather than their parameters. For input $x$, the ensembled output $f_{\text{ens}}(x)$ is typically a weighted average of individual model predictions. For equally ensembled $n$ models:  
\[
f_{\text{ens}}(x) = \frac{1}{n} \sum_{i=1}^n f_{\theta^{(i)}}\left(x\right),
\]  
where $f_{\theta^{(i)}}\left(x\right)$ denotes the output of the $i$-th model. ensembling leverages the ``wisdom of the crowd'' effect, reducing variance and improving robustness~\cite{sagi2018ensemble}. While ensembling outperforms naive parameter averaging~\cite{lv2023parameter}, it requires storing all models and computing their joint outputs, which is infeasible for large-scale deployments.

% Recent work~\cite{zhou2023goingllfc} bridges merging and ensemble by empirically showing that under certain conditions a merged model approximates the ensemble: 
% \[
% \|f_{\theta_{\text{merged}}}(x) - f_{\text{ens}}(x)\| \leq \epsilon \quad \text{for small } \epsilon. 
% \]  
% However, this bound depends heavily on the source models' functional similarity and has been validated only through experiments. 

\subsection{Parameter Pruning in Model Merging}  
\label{ties}
Task vectors from different fine-tuned models could have direction conflicts, degrading the merged model's performance. Pruning resolves parameter conflicts during merging by selectively discarding weights in task vectors. Let the task vector from the $i$-th model be $\Delta^{(i)} = \theta^{(i)} - \theta_{\text{pretrained}}$, \textbf{TIES} ~\cite{tiesyadav} consists of three key stages: trim redundant parameters, elect dominant signs, and merge the aligned parameters.
 \textbf{DARE}~\cite{yu2024languagedare} shows the importance of rescaling after discarding. It replaces step 1 in TIES with randomly keeping $k\%$ parameters and rescaling the kept parameters by $1/k\%$. However, DARE requires the model to be large so that random trimming would not incur feature loss.

Our work introduces M-Loss to guide pruning by directly measuring functional discrepancies between merged and ensembled predictions, eliminating reliance on additional labeled data.

\section{Problem Formulation}
\label{sec:preliminaries}

\subsection{Problem Description}

Consider a set of models $\{M_1, M_2, \dots, M_q\}$ that share the same architecture, consisting of $L+1$ layers with input dimension $m_0$, hidden dimensions $m_1, \dots, m_L$, and output dimension $m_{L+1}$. The neural network function for model $M$ is defined as:
\begin{equation}
    f_M(x) = W^{L+1}_M \sigma\left(H^L_M \dots \sigma\left(H^1_M\left(x\right)\right)\right)+b_M^{L+1}, 
\end{equation}
where $H^k_M\left(x\right) = W_M^k x + b_M^k$, for $k = 1, \dots, L$, represents the linear transformation of the $k$-th layer. Here, $W_M^k$ and $b_M^k$ denote the weights and biases of the $k$-th layer of model $M$, respectively.

The pre-activation vector of the $j$-th layer is defined as:
\begin{equation}
    h_M^j\left(x\right) = H^{j}_M \sigma\left(H^{j-1}_M \dots \sigma\left(H^1_M x\right)\right), 
\end{equation}
for $j = 1, \dots, L$. We write $h_j$ for simplicity if $M$ and $x$ are clear from context. The $i$-th entry is written as $h_M^{j,i}(x)$ or $h_{j,i}$.

The discrepancy between model merging and model ensembling arises mainly from the non-linear activation functions. To address this, we investigate the flow of intermediate representations around the activation functions.

\subsection{Definitions}

To establish our M-Loss theory, we first need to introduce the following definitions: 

\begin{definition}[Intermediate Representation]
\label{intermediaterepresentation}
    The intermediate representation of a given input vector $x$ with respect to a model $M$ refers to the representation formed by partially passing the input vector $x$ through the model (e.g, pre-/post-activation vectors $h_k,x_k$, their entries $h_{k,i},x_{k,i}$, the input and output of hidden layers as shown in Figure 2.
\end{definition}

For example, the intermediate representation $h_k$ passes through the activation function in layer $k$ to obtain the post-activation intermediate representation $x_k$.

\begin{definition}[Linearly Correlated Model Parameters]
\label{linearcor} 
    For any intermediate representation $h_{k,i}$ of a model $M$ characterized by parameter set $\Theta$, the Linearly Correlated Model Parameters (LCP) of $h_{k,i}$ (or of node $N_{k,i}$) are defined as the parameter subset holding constant partial derivative determined by input vector only: $\{w \in \Theta \mid \frac{\partial h_{k,i}}{\partial w} = C \text{, where $C$ is a nonzero constant given by input vector $x$}\}$.
\end{definition}

For example, given a simple layer without bias, denoted as $h_k = Wx$, the LCP with respect to $h_{k,i}$ is the parameters in the $i$-th row of $W$ because $\frac{\partial h_{k,i}}{\partial W_{i,j}} = x_j$. An example is shown in Figure 2. 
\begin{figure}[t]
  \centering
  \includegraphics[width=\columnwidth]{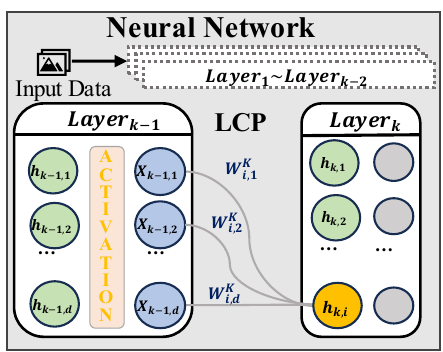}
  \caption{Visualization of Linearly Correlated Parameters (LCP) in a neural network. The figure illustrates how the partial derivative $\partial h_{k,i}/\partial W_{ij}^k$ relates to the input from the previous layer $x_{(k-1),j}$.}
  \label{fig:lcp-concept}
\end{figure}

% \begin{definition}[n-th Order Correlated Model Parameters]
%     For any intermediate representation $h$ of a model $M$ characterized by parameter set $\Theta$, the n-th Order Correlated Model Parameters (n-OCP) of $h$ are defined as the parameter subset $\{w \in \Theta \mid \frac{\partial^n h}{\partial w^n} = C \text{, where $C$ is a constant given a fixed input vector $x$}\}$.
% \end{definition}

\begin{definition}[M-Loss on Layer Level]
    For a given input vector \(x\), pretrained models \(\mathcal{M} = \{M_1, \dots, M_q\}\), merging weights \(\alpha=(\alpha_1, \dots, \alpha_q)\), and the \(j\)-th layer of model \(M\), the merging-ensembling loss (M-Loss) is defined as:
    \[
      h = \bigl(h^j_{M_1}(x), \dots, h^j_{M_q}(x)\bigr),
      \quad
      \alpha\!\cdot\!h = \sum_{p=1}^q \alpha_p\,h^j_{M_p}(x).
    \]
    \begin{equation}
      \mathcal{L}\bigl(W_{\mathcal{M}}^{j}\bigr)(x)
      = \Bigl\lVert
          \sigma(\alpha\!\cdot\!h)
        - \alpha\!\cdot\!\sigma(h)
        \Bigr\rVert_2
    \end{equation}
\end{definition}
where \(\lVert\cdot\rVert_2\) denotes the \(L_2\)-norm.

\begin{remark}
    If the merging weights \(\alpha_i\) are not specified, they are set to \(\tfrac{1}{q}\) by default for simple averaging.
\end{remark}

\begin{definition}[M-Loss on Node Level]
    For a given input vector \(x\), pretrained models \(\mathcal{M} = \{M_1, \dots, M_q\}\), merging weights \(\alpha = (\alpha_1, \dots, \alpha_q)\), and the \(i\)-th node of the \(k\)-th layer in model \(M\), the M-Loss is defined as the absolute value of output discrepancy:
    \[
      h = \bigl(h_{M_1}^{k,i}(x), \dots, h_{M_q}^{k,i}(x)\bigr),
      \quad
      \alpha\!\cdot\!h = \sum_{p=1}^q \alpha_p\,h_{M_p}^{k,i}(x).
    \]
    \begin{equation}
      \mathcal{L}\bigl(N_{k,i}^{\mathcal{M}}\bigr)(x)
      = \Bigl\lvert
          \sigma(\alpha\!\cdot\!h)
        - \alpha\!\cdot\!\sigma(h)
        \Bigr\rvert
    \end{equation}
\end{definition}
We write \(\mathcal{L}(N_{k,i})\) for simplicity if \(\mathcal{M}\) is clear from context.

\begin{remark}[Relation between Node and Layer M-Loss]
\label{remark:nodelayerlevel}
Note that by definition, we have: 
\begin{equation}
 \mathcal{L}\left(W_{\mathcal{M}}^{k}\right)(x) = \sqrt{\sum_{i=1}^{d}\mathcal{L}\left(N_{k,i}^{\mathcal{M}}\right)^2}, 
\end{equation}
where $d$ is the hidden dimension of the model.

\end{remark}

For models without layer normalization \cite{ba2016layernormalization}, we adopt the following normalized M-Loss to avoid the effect of scaling among layers.

\begin{definition}[Normalized M-Loss on Layer Level]
  For a given input \(x\), pretrained models \(\mathcal{M} = \{M_1, \dots, M_q\}\) and merging weights \(\alpha = (\alpha_1,\dots,\alpha_q)\), denote
  \[
    h = \bigl(h^j_{M_1}(x), \dots, h^j_{M_q}(x)\bigr),
    \quad
    \alpha\cdot h = \sum_{p=1}^q \alpha_p\,h^j_{M_p}(x).
  \]
  Then the normalized merging-ensembling loss is
  \begin{equation}
    \mathcal{L}_{\mathrm{norm}}\bigl(W_{\mathcal{M}}^{j}\bigr)(x)
    = \frac{\bigl\|\sigma(\alpha\!\cdot\!h)
          - \alpha\!\cdot\!\sigma(h)\bigr\|_2}
           {\|\sigma(\alpha\!\cdot\!h)\|_2 + \epsilon},
  \end{equation}
where \(\epsilon > 0\) is a small constant added in the denominator to ensure numerical stability (typically \(\epsilon = 10^{-4}\) or smaller).
\end{definition}

\begin{definition}[Normalized M-Loss on Node Level]
  For a given input vector \(x\), pretrained models \(\mathcal{M} = \{M_1, \dots, M_q\}\) and merging weights \(\alpha = (\alpha_1,\dots,\alpha_q)\), denote
  \[
    h = \bigl(h_{M_1}^{k,i}(x), \dots, h_{M_q}^{k,i}(x)\bigr),
    \quad
    \alpha\!\cdot\!h = \sum_{p=1}^q \alpha_p\,h_{M_p}^{k,i}(x).
  \]
  Then the normalized node-level merging-ensembling loss is
  \begin{equation}
    \mathcal{L}_{\mathrm{norm}}\bigl(N_{k,i}^{\mathcal{M}}\bigr)(x)
    = \frac{\bigl|\,
        \sigma(\alpha\!\cdot\!h)
        - \alpha\!\cdot\!\sigma(h)
      \bigr|}
      {\bigl|\sigma(\alpha\!\cdot\!h)\bigr| + \epsilon},
  \end{equation}
  where \(\epsilon > 0\) is a small constant added in the denominator to ensure numerical stability (typically \(\epsilon = 10^{-4}\) or smaller).
\end{definition}

Model merging operates on weights rather than alignment through non-linear layers. Empirical results \cite{jin2022dataless,dai2025leveraging} indicate that linear layers are more significant to the merging process than non-linear layers.
% For further study on non-linear layers such as attention layers~\cite{bahdanau2014neuralattention}, we provide a definition of M-Loss for any intermediate representation in an arbitrary type of layer.
% \begin{definition}[General M-Loss beyond Linear Layers]
%     The M-Loss for an intermediate representation extends the layer and node-level definitions to encompass the distributional divergence between the parameter-averaged approach and the ensemble approach. This can be quantified using metrics such as distributional divergence or Wasserstein distance.
% \end{definition}

\section{M-Loss as a Quantitative Criterion for Theoretical Merging Feasibility}

The central question on model merging feasibility is: why can simple parameter averaging approximate output-level ensembling for non-linear networks? In this section, we address this challenge under the M-Loss theory.

By Remark 2, layer-level M-Loss is from the accumulation of node-level M-Loss. Meanwhile, the discrepancy between model merging and ensembled output is from the accumulation of layer-level M-Loss, as the discrepancy is accumulated when passing through each activation function.
Elementally, we calculate the expected node-level M-Loss to bound the discrepancy between the model merging output and the ensembled output.

% For a theoretical analysis, we simplify the complexity of multiple models and uneven weights by merging two models with equal aggregation weights.

Consider the activation function $f(\cdot)$. We aim to compute the expected M-Loss on the node level as:
\[
\mathbb{E}[D] = \mathbb{E}_{x, a, b} \left[ \left| f\left( \frac{a + b}{2} \right) - \frac{f(a) + f(b)}{2} \right| \right],
\]
where:
\begin{itemize}
    \item \( x \sim \text{Uniform}(-k, k) \): Approximates the \textbf{intermediate representation} of the pre-activation value through the pretrained model.
    \item \( a, b \sim \mathcal{N}(x, \sigma^2) \): Approximate \textbf{intermediate representation} of pre-activation values through fine-tuned models, which correspond to two different $h_{k,i}$ in the previous section
    \item Assume that $k \gg \sigma$, as fine-tuned models have a small amount of parameter shift from the pre-trained model.
\end{itemize}

For the three commonly used activation functions: ReLU~\cite{nair2010rectifiedrelu}, GELU~\cite{hendrycks2016gaussiangelu}, and Leaky ReLU~\cite{maas2013rectifierleaky}, we have the following estimation (see Appendix for detailed derivations):
\[
\boxed{\mathbb{E}[D_{\mathrm{ReLU}}]  \approx \frac{\sigma^2}{\sqrt{2}\pi k}}
\]
\[
\boxed{\mathbb{E}[D_{\mathrm{Leaky ReLU,\alpha}}]  \approx \frac{(1-\alpha)\sigma^2}{\sqrt{2}\pi k}}
\]
\[
\boxed{\mathbb{E}[D_{\mathrm{GELU}}]  \approx \frac{\sigma^2}{4k}}
\]
As $k$ is the range of intermediate representation of the pretrained model, estimation of M-Loss indicates that merging models fine-tuned from a more general pretrained model (larger range $k$) would perform better. Meanwhile, $\sigma$ measures the expected parameter difference of source models, indicating that mergeeability is positively correlated with parameter similarity.

When the assumption $k \gg \sigma$ holds, the expectations calculated above are all small, so the expectation of M-Loss for merging fine-tuned models from the same pretrained model is small. Thus, we theoretically proved that the merged models obtained by averaging parameters behave similarly to averaging the outputs of each individual model. 

% \subsection{Proof Sketch}

% Here, we provide a sample proof sketch for the expectation of ReLU M-Loss. Note that the complete proof for the three activation functions is presented in Appendix~\ref{sec:appendix_proof}. 

% \begin{enumerate}
%     \item \textbf{Formulation of M-Loss:} The discrepancy $D_{\mathrm{ReLU}}$ is defined based on the midpoint of $a$ and $b$ compared to their individual ReLU outputs. Due to ReLU's non-linearity, $D$ is nonzero only when $a$ and $b$ are on opposite sides of zero.
    
%     \item \textbf{Probability of Nonzero M-Loss:} This event occurs with probability $2 \Phi(x/\sigma) \Phi(-x/\sigma)$, where $\Phi(\cdot)$ is the CDF of a standard normal distribution.
    
%     \item \textbf{Conditional Expectation of M-Loss:} Given that $a \geq 0$ and $b < 0$, the expected discrepancy is proportional to the expected absolute deviation of a Gaussian variable, yielding $\mathbb{E}[D | x] = \frac{\sigma}{\sqrt{2\pi}}$.
    
%     \item \textbf{Integration Over $x$:} Taking the expectation over $x$, the integral $\int_{-k}^{k} \Phi(x/\sigma) \Phi(-x/\sigma) dx$ is approximated by extending the limits to infinity (justified by $\sigma \ll k$). The resulting integral evaluates to $1/\sqrt{\pi}$.
    
%     \item \textbf{Final Approximation:} Substituting this result, we derive the final expression:
%     \[
%         \mathbb{E}[D_{\mathrm{ReLU}}] \approx \frac{\sigma^2}{\sqrt{2} \pi k}.
%     \]
% \end{enumerate}

\begin{algorithm}[tb]
   \caption{M-TIES Merging}
   \label{alg:layerwise_mt_merge}
\begin{algorithmic}[1]
   \STATE \textbf{Input:} Pretrained model $M$ with parameters $\theta$ consisting of $L$ layers $\theta^{(1)},\dots,\theta^{(L)}$; source models $M_1,\dots,M_q$ with parameters $\theta_1,\dots,\theta_q$; merging weights $\alpha_1,\dots,\alpha_q$; unlabeled dataset $\mathcal{D}$; base keep percentage $k\%$; pruning variation $e\%$.
   \STATE \textbf{Output:} Merged model $M_{\mathrm{merged}}$.
   \STATE Initialize $M_{\mathrm{merged}} \leftarrow M$.
   \FOR{$l = 1$ \TO $L$}
       \STATE \# Compute task vectors for layer $l$
       \FOR{$i = 1$ \TO $q$}
           \STATE $T_i^{(l)} \leftarrow \theta_i^{(l)} - \theta^{(l)}$
       \ENDFOR
       \STATE \# Compute node-wise M-Loss for layer $l$
       \FOR{$j = 1$ \TO $d_l$}
           \STATE Compute $\mathcal{L}(N_{l,j})$ on layer $l$.
       \ENDFOR
       \STATE \# Determine keep ratios based on M-Loss
       \FOR{$j = 1$ \TO $d_l$}
           \STATE $k_{l,j} \leftarrow \mathrm{RankNorm}\bigl(\mathcal{L}(N_{l,j})\bigr) \in [(k-e)\%,\,k\%]$
       \ENDFOR
       \STATE \# Prune task vectors at node level
       \FOR{$j = 1$ \TO $d_l$}
           \STATE Prune parameters in $\{T_i^{(l)}\}_{i=1}^q$ linearly correlated with $N_{l,j}$, keeping proportion $k_{l,j}$.
       \ENDFOR
       \STATE \# Elect sign and merge pruned task vectors
       \STATE $\Delta\theta^{(l)} \leftarrow \sum_{i=1}^q \alpha_i\,\mathrm{ElectPruned}(T_i^{(l)})$
       \STATE Update layer $l$: $\theta^{(l)} \leftarrow \theta^{(l)} + \Delta\theta^{(l)}$ in $M_{\mathrm{merged}}$.
       \STATE \# Synchronize merged layer to source models
       \FOR{$i = 1$ \TO $q$}
           \STATE $\theta_i^{(l)} \leftarrow \theta^{(l)}$
       \ENDFOR
   \ENDFOR
   \RETURN $M_{\mathrm{merged}}$.
\end{algorithmic}
\end{algorithm}

\section{M-Loss as an Indicator of Hyperparameter Selection}

\begin{table*}[t]
\centering

\setlength{\tabcolsep}{1.5mm} % 稍微调整列间距以适应新的列
\begin{tabular*}{\textwidth}{@{\extracolsep{\fill}}llrrrrrrrrr}
\toprule
\textbf{Backbone} & \textbf{Method} & \textbf{RESISC45} & \textbf{Cars} & \textbf{MNIST} & \textbf{DTD} & \textbf{EuroSAT} & \textbf{GTSRB} & \textbf{SUN397} & \textbf{SVHN} & \textbf{Avg} \\
\midrule
\multirow{6}{*}{ViT-B/32} 
& M-TIES           & \underline{72.60} & 61.07              & 97.62              & \underline{54.84} & \underline{82.02} & \underline{72.44} & 62.19              & 83.06              & \underline{73.23} \\
& TIES             & 70.67              & 58.61              & 98.30              & 54.20              & 80.22              & 72.11              & 59.01              & \textbf{86.20}     & 72.42             \\
& Task Arithmetic  & 71.27              & 60.70              & 95.32              & 51.76              & 79.74              & 67.32              & 62.06              & 76.68              & 70.61             \\
& Simple Avg.      & 71.46              & \underline{63.34}  & 87.46              & 50.11              & 73.00              & 52.79              & \underline{64.91}  & 64.16              & 65.90             \\
& DARE             & 69.97              & 57.98              & \underline{97.95}  & 53.24              & 78.89              & 72.00              & 59.14              & \underline{83.96}  & 71.64             \\
& Ensembling         & \textbf{79.87}     & \textbf{66.60}     & 95.80              & \textbf{58.30}     & \textbf{98.30}     & \textbf{81.11}     & \textbf{66.35}     & 82.15              & \textbf{78.56}    \\
\midrule
\multirow{6}{*}{ViT-L/14} 
& M-TIES           & \textbf{88.57}     & \underline{83.35}  & \textbf{99.06}     & \textbf{66.91}     & \underline{94.61}  & 83.80              & \textbf{76.13}     & \textbf{89.78}     & \underline{85.28}  \\
& TIES             & 88.19              & 82.81              & \underline{99.01}  & 66.70              & 94.37              & 83.36              & \underline{75.65}  & \underline{89.42}  & 84.94             \\
& Task Arithmetic  & 86.17              & 82.44              & 98.54              & 65.59              & 93.93              & 83.47              & 73.56              & 85.26              & 83.62             \\
& Simple Avg.      & 82.67              & 81.54              & 97.01              & 62.77              & 91.17              & 70.63              & 71.65              & 78.23              & 79.46             \\
& DARE             & \underline{88.33}  & \underline{83.35}  & 98.97              & \underline{66.86}  & 94.06              & \underline{84.20}  & 75.37              & 89.19              & 85.04             \\
& Ensembling         & 87.73              & \textbf{85.36}     & 98.78              & 66.81              & \textbf{98.24}     & \textbf{87.92}     & 74.76              & 84.92              & \textbf{85.56}    \\
\bottomrule
\vspace{-15pt}

\end{tabular*}
\caption{Accuracy comparison of merging methods on ViT-B/32 and ViT-L/14 backbones. For each backbone and each dataset, the best result is in bold and the second-best is underlined. M-TIES consistently outperforms other merging methods and achieves competitive results against the much more costly Ensembling baseline, especially on the larger ViT-L/14 model.}
\label{tab:main_results_merged}
\end{table*}

We show that M-Loss can act as an indicator of parameter conflicts between models, guiding the pruning budget schedule of other merging methods such as TIES. By prioritizing parameters with lower M-Loss for retention, we develop a dynamic pruning rate scheduler as a plug-in to the existing merging method TIES~\cite{tiesyadav}.

Different from the TIES method with a fixed trimming rate, we propose an enhanced pruning strategy inspired by TIES merging, leveraging the M-Loss metric. We determine the pruning (trimming) proportion in a more dedicated way, computed from the M-Loss.

\noindent\textbf{M-TIES Merging} receives a pretrained base model \(M\) (parameters \(\theta\)), \(q\) source models \(M_1,\dots,M_q\) (parameters \(\theta_1,\dots,\theta_q\)), merging weights \(\alpha_1,\dots,\alpha_q\), an unlabeled dataset \(\mathcal{D}\), a base keep-rate \(k\%\) and a pruning variation \(e\%\). It then processes each layer \(l=1,\dots,L\) in turn: first computing task vectors \(T_i^{(l)}=\theta_i^{(l)}-\theta^{(l)}\) for all sources; next evaluating a node-level M-Loss \(\mathcal{L}(N_{l,j})\) on \(\mathcal{D}\); rank-normalizing these losses into \([(k-e)\%,k\%]\) to obtain per-node keep ratios \(k_{l,j}\); pruning each \(T_i^{(l)}\) by zeroing out the smallest in absolute value \((100-k_{l,j})\%\) of parameters linearly correlated with node \(j\); then elect the dominant sign of each entry, forming a weighted merge increment \(\Delta\theta^{(l)}=\sum_i\alpha_i\,\mathrm{ElectPruned}(T_i^{(l)})\) which is added back to the base layer: \(\theta^{(l)}\!\leftarrow\!\theta^{(l)}+\Delta\theta^{(l)}\). Finally, the updated \(\theta^{(l)}\) is synchronized to all sources before proceeding to the next layer. After all layers are merged, the resulting model \(M_{\mathrm{merged}}\) is returned.
This method effectively balances the integration of multiple models while selectively pruning redundant or conflicting parameters, ensuring an optimal trade-off between model performance and parameter efficiency.

% \section{Experiments}

% \begin{table*}[t]
% \centering
% \setlength{\tabcolsep}{1mm} % 压缩列间距
% \begin{tabular*}{\textwidth}{@{\extracolsep{\fill}}lrrrrrrrrr}
% \toprule
% Method          & RESISC45       & Cars           & MNIST          & DTD            & EuroSAT        & GTSRB          & SUN397         & SVHN           & Avg            \\
% \midrule
% TIES            & 70.67          & 58.61          & 98.30          & 54.20          & 80.22          & 72.11          & 59.01          & \textbf{86.20} & 72.42          \\
% Layerwise-TIES  & 70.98          & 58.57          & \textbf{98.36} & 54.79          & 81.46          & 70.85          & 58.08          & 86.16          & 72.41          \\
% M-TIES          & \textbf{72.60} & \textbf{61.07} & 97.62          & \textbf{54.84} & \textbf{82.02} & \textbf{72.44} & \textbf{62.19} & 83.06          & \textbf{73.23} \\
% \bottomrule
% \vspace{-15pt}
% \end{tabular*}
% \caption{Ablation study comparing TIES, Layerwise‑TIES, and M‑TIES on ViT‑B/32 models.}
% \label{layerwiseablation}
% \end{table*}

\begin{table*}[t]
\centering
\setlength{\tabcolsep}{1mm} % 压缩列间距以适应宽度
\begin{tabular*}{\textwidth}{@{\extracolsep{\fill}}lrrrrrrrrr}
\toprule
Sample Size & RESISC45 & Cars   & MNIST  & DTD    & EuroSAT & GTSRB  & SUN397 & SVHN   & Avg    \\
\midrule
128         & 72.603   & 61.074 & 97.620 & 54.840 & 82.019  & 72.439 & 62.195 & 83.063 & 73.232 \\
256         & 72.571   & 61.074 & 97.620 & 54.840 & 82.000  & 72.462 & 62.186 & 83.059 & 73.227 \\
512         & 72.571   & 61.074 & 97.620 & 54.840 & 82.019  & 72.454 & 62.176 & 83.059 & 73.227 \\
\bottomrule

\end{tabular*}
\caption{Accuracy of M‑TIES merging ViT‑B/32 fine‑tuned models with different sample sizes for computing M‑loss.}
\label{samplesize}
\end{table*}

\begin{table*}[t]
\centering
\setlength{\tabcolsep}{1mm} % 适度压缩列间距
\begin{tabular*}{\textwidth}{@{\extracolsep{\fill}}lrrrrrrrrr}
\toprule
Random Seed & RESISC45 & Cars & MNIST & DTD & EuroSAT & GTSRB & SUN397 & SVHN & Avg \\
\midrule
1  & 72.603 & 61.074 & 97.62 & 54.840 & 82.278 & 72.478 & 61.855 & 83.063 & 73.225 \\
2  & 72.587 & 61.099 & 97.61 & 54.840 & 81.611 & 72.486 & 62.025 & 83.059 & 73.165 \\
42 & 72.603 & 61.074 & 97.62 & 54.840 & 82.019 & 72.439 & 62.195 & 83.063 & 73.232 \\
\midrule
Std & 0.009 & 0.014 & 0.006 & 0.000 & 0.336 & 0.025 & 0.170 & 0.002 & 0.037 \\
\bottomrule

\end{tabular*}
\caption{Accuracy and standard deviation of M-TIES merging ViT-B/32 fine-tuned models under different random seeds.}
\label{vitb32_seed_variance}
\end{table*}

\begin{table*}[t]
\centering
\setlength{\tabcolsep}{1mm} % 适度压缩列间距
\begin{tabular*}{\textwidth}{@{\extracolsep{\fill}}lrrrrrr}
\toprule
Method    & Flowers102 & FashionMNIST & Food101 & STL10  & CIFAR100 & Avg    \\
\midrule
M‑TIES    & \underline{54.82} & \underline{62.94} & \underline{61.02} & \underline{89.90} & \underline{48.60} & \underline{63.46} \\
TIES      & 48.40             & 58.61             & 52.35             & 87.60             & 40.78             & 57.55             \\
DARE      & 51.86             & 62.28             & 55.14             & 87.80             & 44.19             & 60.25             \\
Ensembling  & \textbf{62.77}    & \textbf{67.93}    & \textbf{72.61}    & \textbf{96.00}    & \textbf{61.70}    & \textbf{72.20}    \\
\bottomrule

\end{tabular*}
\caption{Accuracy comparison of different merging methods on out‑of‑domain datasets, using ViT‑B/32 models.}
\label{ood}
\end{table*}

\section{Experiments}

\subsection{Experimental Setup}
To validate the effectiveness of the proposed M-TIES method, we conduct model merging on Vision Transformer (ViT)~\cite{alexey2020imagevit} with pretrained OpenAI CLIP~\cite{radford2021learningclip} ViT-B/32 and ViT-L-14 models. We use the open-source model checkpoints fine-tuned on eight datasets of diverse types of tasks -- RESISC45~\cite{resisccheng2017remote}, Cars~\cite{Carskrause20133d}, MNIST~\cite{lecun1998gradientmnist}, DTD~\cite{cimpoi2014describingdtd}, EuroSAT~\cite{helber2019eurosat}, GTSRB~\cite{stallkamp2012mangtsrb}, SUN397~\cite{xiao2010sun397} and SVHN~\cite{svhnnetzer2011reading} as source models. The aggregation weights $\alpha_i$'s are set to be equal. We compare our M-TIES method with other model merging baselines, namely \textbf{Simple Average}~\cite{dash2004model,hoeting1999bayesian,wortsman2022modelsoups}, ~\textbf{Task Arithmetic}~\cite{ilharco2022taskarith}, ~\textbf{TIES}~\cite{tiesyadav} and \textbf{DARE}~\cite{yu2024languagedare}. We discuss in the Appendix for baseline selection. For the unlabeled dataset selection, we only use 128 samples in total, which means 16 samples for each dataset on average. This is a very few-shot setting and is practical in real-world scenarios. For other layers, as they are less important in merging~\cite{jin2022dataless,dai2025leveraging} and require a more complex procedure for computing M-Loss than linear layers, we adopt vanilla trimming with TIES. We conduct all experiments using a single NVIDIA RTX A6000 GPU.

% \subsection{M-Loss as an Indicator of Mergeability}

% To evaluate the effectiveness of M-Loss as an indicator of the mergeability of different models, we perform the following experiment:

% \begin{enumerate}
%     \item \textbf{Dataset Splitting}: Split a benchmark dataset into \( n \) non-overlapping subsets.
%     \item \textbf{Model Fine-tuning}: For each subset, fine-tune a pretrained base model to obtain \( n \) distinct source models.
%     \item \textbf{Model Pair Merging}: Merge each pair of source models using our M-Loss guided pruning strategy.
%     \item \textbf{Performance Evaluation}: For each merged model, evaluate its accuracy on a held-out test set and compare it to the average accuracy of the two source models.
%     \item \textbf{Correlation Analysis}: Compute the M-Loss for each pair of source models and analyze the correlation between M-Loss values and the corresponding accuracy improvements.
% \end{enumerate}

% \subsection{M-Loss as an Indicator of Parameter Significance}

% % \subsubsection{Experiment Design}

% We demonstrate that M-Loss can act as an indicator of parameter significance, guiding the parameter budget schedule of other merging methods such as TIES. The experiment involves ranking parameters based on their M-Loss values and using this ranking to prioritize which parameters to retain during the pruning process.

% \section{Experiments on TIES}
% \section{Experimental Setup}

\subsection{Experimental Results}

The experimental results for our method, four merging baselines, and ensembling are shown in Table 1, covering experiments on both the ViT-B/32 and the larger ViT-L/14 model. For each method except for simple averaging, we first conduct a hyperparameter search on the validation set using the ViT-B/32 model and report the best average accuracy. The hyperparameter search range is included in the Appendix. We then repeat the experiments on the ViT-L/14 model using these same fixed hyperparameters. As explained in the table's caption, for each backbone, numbers in \textbf{bold} are the top accuracies for the task, and \underline{underlined} numbers denote the second best.

% We present the experimental results of four merging baselines, ensemble, and our method with the ViT-B/32 model in Table 1. For each method except for simple averaging, we conduct a hyperparameter search on the validation set and report the best average accuracy. The hyperparameter search range is included in the Appendix. We also repeat the experiments using the ViT-L/14 model with a larger model size as in Table 2 using fixed hyperparameters. Note that in both Table 1 and Table 2, numbers in bold are the top accuracies for the task, and underlined numbers denote the second best. % for a validation on model with larger size, with a fixed hyperparameter representing the case where no validation set is available.

\subsection{Experiment Results Analysis}

\subsubsection{General Comparison}
From Table 1, we can see that our proposed M-TIES method outperforms other merging baselines in both ViT-B/32 and ViT-L/14 model backbones, while below that of model ensembling. This verifies our motivation to bridge the gap between merging and ensembling by M-Loss. Meanwhile, calculation gives the variance of accuracy for top-3 merging methods M-TIES, TIES, and DARE are 172.22 and 203.27. and 197.61 (in ViT-B/32) respectively, which shows that M-TIES holds better stability and fairness (not favoring high-accuracy tasks). Additionally, when the accuracy gap between merging and ensembling is small, M-TIES even outperforms ensembling for certain tasks.

\begin{table*}[t]
\centering
\setlength{\tabcolsep}{1mm} % 适度压缩列间距
\begin{tabular*}{\textwidth}{@{\extracolsep{\fill}}lrrrrrrrrr}
\toprule
Layer Index       & RESISC45 & Cars    & MNIST  & DTD    & EuroSAT & GTSRB  & SUN397 & SVHN   & Avg    \\
\midrule
All                & 72.603   & 61.074  & 97.620 & 54.840 & 82.019  & 72.439 & 62.195 & 83.063 & 73.232 \\
0,8,9,10          & 72.619   & 61.112  & 97.570 & 55.000 & 82.000  & 72.328 & 62.337 & 82.760 & 73.216 \\
8,9,10            & 72.619   & 61.099  & 97.560 & 55.000 & 82.074  & 72.312 & 62.333 & 82.771 & 73.221 \\
\bottomrule
\vspace{-15pt}
\end{tabular*}
\caption{Accuracy of M‑TIES merging ViT‑B/32 fine‑tuned models with different layers pruned by M‑loss.}
\label{vitb32_layers}
\end{table*}

\begin{table*}[t]
\centering
\setlength{\tabcolsep}{1mm} % 适度压缩列间距
\begin{tabular*}{\textwidth}{@{\extracolsep{\fill}}lrrrrrrrrr}
\toprule
Layer Index       & RESISC45 & Cars    & MNIST  & DTD    & EuroSAT & GTSRB  & SUN397 & SVHN   & Avg    \\
\midrule
All                & 88.571   & 83.348  & 99.060 & 66.915 & 94.611  & 83.800 & 76.134 & 89.782 & 85.278 \\
0,\,20,\,21,\,22   & 88.587   & 83.273  & 99.070 & 66.809 & 94.593  & 83.903 & 76.088 & 89.747 & 85.259 \\
20,\,21,\,22       & 88.571   & 83.273  & 99.070 & 66.809 & 94.611  & 83.895 & 76.065 & 89.751 & 85.256 \\
\bottomrule
\end{tabular*}
\caption{Accuracy of M‑TIES merging ViT‑L/14 fine‑tuned models with different layers pruned by M‑loss.}
\label{vitl14_layers}
\end{table*}

\subsubsection{Computational Cost Analysis}
As in our algorithm, when merging layer $k$, the previous $(k-1)$ models would have been updated to the same merged weight, so we only need to pass the input to the previous layer once and get the shared input of layer $k$, then compute the M-Loss for nodes in layer $k$. The inference cost is passing input to $(k-1+q)$ layers, which is much less than passing the input to each model individually, which needs $kq$ layers ($q$ is the number of source models).

For the empirical time consumption, M-TIES does not add a significant time cost. M-TIES uses approximately 1 minute and 30 seconds and 3 minutes for merging ViT-B/32 and ViT-L/14 models, respectively, compared with the TIES method, which takes 30 seconds and 1 minute, as the computation of M-Loss only involves forward inference of a small unlabeled dataset. The main time cost for the experiment is the evaluation process, which takes around 5 minutes and 15 minutes for ViT-B/32 and ViT-L/14, respectively.

% \subsection{Sensitivity Analysis}

% Hyperparameter $e$ determines the range of dynamic retention rate $k$ computed from M-Loss to ensure that the retention rate for a row is within the range of $[k-e,k]$. In this section, we show that the accuracy of our M-TIES method is insensitive to the selection of hyperparameter $e$. 

% We conduct merging experiments of M-TIES on the ViT-B/32 model, with fixed $k=0.2$ and $e$ selected from $\{0.05, 0.075, 0.1, 0.125, 0.15\}$. As shown in \cref{fig:sensitivity}, the average accuracy's fluctuation is within 0.1\%, with the highest on $e = \frac{k}{2}$. We conclude that the choice of hyperparameter $e$ is not crucial to the merged model performance, provided that $e$ is not too small (e.g., $e = 0$, resulting in the same trimming process as TIES). 

% \begin{figure}[htpb]
% \centering
% \includegraphics[width=0.5\linewidth]{AAAI_AuthorKit26/images/k_e.pdf}
% \caption{Performance stability analysis with varying $\varepsilon$-neighborhood parameter $e$ and fixed $k=0.2$. The classification accuracy remains relatively stable ($77.107\%$-$77.203\%$) across different $e$ values, demonstrating the robustness of the proposed method to this hyperparameter.}
% \label{fig:sensitivity}
% \end{figure}

\begin{figure}[t]
\centering
\includegraphics[width=0.4\textwidth]{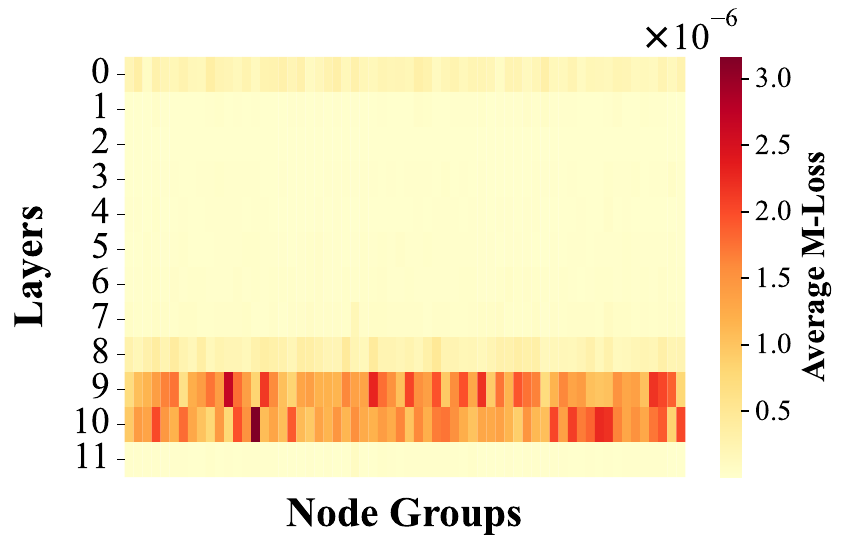}
\caption{Layerwise node group M-Loss distribution across different layers of ViT-B/32 models. Each colored block reveals the average M-Loss of 50 consecutive nodes, with the x-axis being the node group number and the y-axis being the layer number.}
\label{fig:mloss_heatmap}
\end{figure}

% \subsection{Ablation Study on Layerwise TIES}

% As our M-TIES compares parameter magnitudes on a layerwise basis, we study the case where the trimming step in TIES is not taken by ranking and keeping top-$k\%$ absolute value parameters globally, but does the ranking inside each layer, similar to M-TIES, namely Layerwise TIES. As shown in Table 3, adopting the layerwise method has no performance gain.

% As shown in Table 3, Layerwise TIES has no accuracy increase compared with vanilla TIES, while our M-TIES method has an increase of 0.81\%. This shows that improvement of M-TIES does not result from changing the global ranking top-k\% to layerwise ranking top-k\% in the trimming procedure.

\subsection{Sampling Efficiency and Stability over Randomness in Evaluating M-Loss}

Table 2 shows that computing M-Loss only requires very limited unlabeled data. As we vary the sample size for evaluating the M-Loss, we find that there is little change in the final results. Thus, we chose only one batch of 128 samples in our experiments. For the randomness test, we run the experiments over three random seeds for drawing 128 samples as in Table 3, with a low standard deviation reported. It shows that our method holds great stability over random sampling.

\subsection{Out-of-Domain Generalizability of M-TIES}
We also conduct out-of-domain testing, by evaluating the ensembled model and top-3 ViT-B/32 models merged as in Table 1 on Flowers 102~\cite{flowers102}, FashionMNIST~\cite{xiao2017fashion}, Food 101~\cite{bossard2014food101}, STL 10~\cite{STLpmlr-v15-coates11a}, and CIFAR100~\cite{cifar100krizhevsky2009learning}. As shown in Table 4, M-TIES outperforms other baselines except for ensembling, which shows its great generalizability on unseen tasks.

\subsection{Few-Layer M-TIES: a More Efficient Alternative}
We also analyze the M-Loss value across different nodes in different layers during the merging process in Figure 3. We find that the first layer and the last few layers (excluding the output layer) hold higher M-Loss than other layers.

With this insight, we propose Few-Layer M-TIES, which only uses a dynamic pruning method for the first and last few layers, so as to save the computational cost. Experiment results in Table 5 and Table 6 show that there is almost no change to the performance, but the computational cost is largely reduced.

\section{Conclusion}

This paper introduces \textbf{M-Loss}, a novel metric that quantifies the gap between parameter-averaged and output-averaged models without relying on labeled data. By computing the expected M-Loss for common activation functions, we provide theoretical justification for the conditions where parameter averaging yields predictions close to model ensembling, thereby establishing theoretical model merging feasibility. To show that M-Loss can be integrated with a concrete merging method, we integrate the M-Loss dynamic budget scheduler into the TIES merging framework, guiding the selective removal of conflicting parameters. The integration leads to superior performance compared to existing methods. Empirical evaluation results on ViT models underscore M-Loss's effectiveness in identifying crucial parameters. Overall, this work advances the theoretical foundations of model merging and contributes practical tools for the efficient merging of multiple models.

\section*{Acknowledgments}
This research is supported by the NTU startup grant and the RIE2025 Industry Alignment Fund –
Industry Collaboration Projects (IAF-ICP) (Award I2301E0026), administered by A*STAR, as well
as supported by Alibaba Group and NTU Singapore through Alibaba-NTU Global e-Sustainability
CorpLab (ANGEL). This research/project is supported by A*STAR under its Japan-Singapore
Joint Call: JST-A*STAR 2024 (Project ID: R24I6IR139).

\bibliography{aaai2026}

\setcounter{secnumdepth}{0}
\renewcommand\thesubsection{\arabic{subsection}}
\renewcommand\labelenumi{\thesubsection.\arabic{enumi}}

\newcounter{checksubsection}
\newcounter{checkitem}[checksubsection]

\newcommand{\checksubsection}[1]{%
  \refstepcounter{checksubsection}%
  \paragraph{\arabic{checksubsection}. #1}%
  \setcounter{checkitem}{0}%
}

\newcommand{\checkitem}{%
  \refstepcounter{checkitem}%
  \item[\arabic{checksubsection}.\arabic{checkitem}.]%
}
\newcommand{\question}[2]{\normalcolor\checkitem #1 #2 \color{blue}}
\newcommand{\ifyespoints}[1]{\makebox[0pt][l]{\hspace{-15pt}\normalcolor #1}}

\appendix
\onecolumn
\section{Proof for Expectation of M-loss}
\label{sec:appendix_proof}

To calculate the expectation of M-loss, first we need to prove the following lemma:
\begin{lemma}[Integral of Product of Normal CDFs]
\label{lemma:phi_integral}
Let \(\Phi(u)\) denote the cumulative distribution function (CDF) of the standard normal distribution:
\[
\Phi(u) = \frac{1}{\sqrt{2\pi}} \int_{-\infty}^u e^{-t^2/2} \, dt.
\]
Then the integral of the product \(\Phi(u)\Phi(-u)\) over the real line is:
\[
\int_{-\infty}^\infty \Phi(u)\Phi(-u) \, du = \frac{1}{\sqrt{\pi}}.
\]
\end{lemma}

\begin{proof}
We prove this result through probabilistic interpretation and symmetry.

\textbf{Step 1: Symmetry Property.}  
Note that \(\Phi(-u) = 1 - \Phi(u)\), so the integrand becomes:
\[
\Phi(u)\Phi(-u) = \Phi(u) - \Phi(u)^2.
\]

\textbf{Step 2: Probabilistic Interpretation.}  
Let \(X\) and \(Y\) be independent standard normal random variables. Then:
\[
\Phi(u) = P(X \leq u), \quad \Phi(-u) = P(Y \geq u).
\]
The product \(\Phi(u)\Phi(-u)\) represents the joint probability:
\[
P(X \leq u \leq Y).
\]

\textbf{Step 3: Integral as Expected Value.}  
The integral can be rewritten as:
\[
\int_{-\infty}^\infty P(X \leq u \leq Y) \, du = E\left[ \int_{-\infty}^\infty I(X \leq u \leq Y) \, du \right],
\]
where \(I(\cdot)\) is the indicator function. For fixed \(X\) and \(Y\), the inner integral equals \( \max(Y - X, 0) \). Thus:
\[
\int_{-\infty}^\infty \Phi(u)\Phi(-u) \, du = E\left[ \max(Y - X, 0) \right].
\]

\textbf{Step 4: Distribution of \(Y - X\).}  
Since \(X\) and \(Y\) are independent \( \mathcal{N}(0,1) \), the difference \(Z = Y - X\) follows \( \mathcal{N}(0, 2) \). Therefore:
\[
E\left[ \max(Z, 0) \right] = \frac{1}{2\sqrt{\pi}} \int_0^\infty z e^{-z^2/4} \, dz.
\]

\textbf{Step 5: Evaluate the Integral.}  
Substitute \(t = z^2/4\), \(dt = z/2 \, dz\):
\[
\int_0^\infty z e^{-z^2/4} \, dz = 2 \int_0^\infty e^{-t} \, dt = 2.
\]
Thus:
\[
E\left[ \max(Z, 0) \right] = \frac{1}{2\sqrt{\pi}} \cdot 2 = \frac{1}{\sqrt{\pi}}.
\]
\end{proof}

\subsection{ReLU Activation Function}

Consider the ReLU activation function \( f(z) = \max(0, z) \). We aim to compute the expected M-loss:
\[
\mathbb{E}[D_{\mathrm{ReLU}}] = \mathbb{E}_{x, a, b} \left[ \left| f\left( \frac{a + b}{2} \right) - \frac{f(a) + f(b)}{2} \right| \right],
\]
where:
\begin{itemize}
    \item \( x \sim \text{Uniform}(-k, k) \): Approximates the embedding through the pretrained model.
    \item \( a, b \sim \mathcal{N}(x, \sigma^2) \): Approximate embeddings through fine-tuned models, centered around \( x \).
    \item Assume $k \gg \sigma$, as fine-tuned models have small amount shift from pre-trained model.
\end{itemize}

\textbf{Step 1: Variable Transformation}

Define:
\[
\delta_a = a - x \sim \mathcal{N}(0, \sigma^2), \quad \delta_b = b - x \sim \mathcal{N}(0, \sigma^2).
\]
Thus,
\[
s = \frac{a + b}{2} = x + \frac{\delta_a + \delta_b}{2}.
\]

\textbf{Step 2: Expression for M-loss (\( D \))}

Under the ReLU activation, the M-loss becomes:
\[
D = \left| \max\left(0, x + \frac{\delta_a + \delta_b}{2}\right) - \frac{\max\left(0, x + \delta_a\right) + \max\left(0, x + \delta_b\right)}{2} \right|.
\]

\textbf{Step 3: Case Analysis}

The discrepancy \( D \) is non-zero only when \( a \) and \( b \) are on opposite sides of zero. Specifically:
\begin{enumerate}
    \item \textbf{Case 1}: Both \( a \geq 0 \) and \( b \geq 0 \) \(\Rightarrow\) \( D = 0 \).
    \item \textbf{Case 2}: Both \( a < 0 \) and \( b < 0 \) \(\Rightarrow\) \( D = 0 \).
    \item \textbf{Case 3}: \( a \geq 0 \) and \( b < 0 \) (denoted as Case A) \(\Rightarrow\) \( D > 0 \).
    \item \textbf{Case 4}: \( a < 0 \) and \( b \geq 0 \) (symmetrical to Case A) \(\Rightarrow\) \( D > 0 \).
\end{enumerate}

Due to symmetry, Cases A and A' (Cases 3 and 4) are identical in their contributions to \( D \).

\textbf{Step 4: Probability of Non-Zero \( D \)}

The probability that \( D > 0 \) given \( x \) is:
\[
P(D > 0 | x) = 2 \times P(a \geq 0, b < 0 | x) = 2 \times \Phi\left(\frac{x}{\sigma}\right) \times \Phi\left(-\frac{x}{\sigma}\right),
\]
where \( \Phi(\cdot) \) is the cumulative distribution function (CDF) of the standard normal distribution.

\textbf{Step 5: Expectation of \( D \) Given \( x \)}

In Case A (\( a \geq 0 \) and \( b < 0 \)):
\[
D = \frac{|x + \delta_a|}{2}.
\]
Thus, by the symmetric distribution of $x$:
\[
\mathbb{E}[D | x, \text{Case A}] = \frac{\mathbb{E}[|\delta_a|]}{2} = \frac{\sigma \sqrt{\frac{2}{\pi}}}{2} = \frac{\sigma}{\sqrt{2\pi}}.
\]

\textbf{Step 6: Total Expectation Over \( x \)}

Taking into account both Cases A and A' (due to symmetry):
\[
\mathbb{E}[D_{\mathrm{ReLU}}] = \mathbb{E}_{x} \left[ \mathbb{E}[D | x] \right] = \mathbb{E}_{x} \left[ \frac{\sigma}{\sqrt{2\pi}} \times 2 \times \Phi\left(\frac{x}{\sigma}\right) \Phi\left(-\frac{x}{\sigma}\right) \right].
\]
Simplifying:
\[
\mathbb{E}[D_{\mathrm{ReLU}}] = \frac{\sigma}{\sqrt{2\pi}} \times \frac{1}{2k} \int_{-k}^{k} 2 \Phi\left(\frac{x}{\sigma}\right) \Phi\left(-\frac{x}{\sigma}\right) dx = \frac{\sigma}{\sqrt{2\pi} \times k} \int_{-k}^{k} \Phi\left(\frac{x}{\sigma}\right) \Phi\left(-\frac{x}{\sigma}\right) dx.
\]

\textbf{Step 7: Approximating the Integral}

As \( \sigma \ll k \), the significant contribution to the integral comes from the vicinity of \( x = 0 \). Therefore, we approximate:
\[
\int_{-k}^{k} \Phi\left(\frac{x}{\sigma}\right) \Phi\left(-\frac{x}{\sigma}\right) dx \approx \int_{-\infty}^{\infty} \Phi\left(\frac{x}{\sigma}\right) \Phi\left(-\frac{x}{\sigma}\right) dx.
\]
Let \( u = \frac{x}{\sigma} \), hence \( dx = \sigma du \):
\[
\int_{-\infty}^{\infty} \Phi(u) \Phi(-u) \sigma du = \sigma \int_{-\infty}^{\infty} \Phi(u) \Phi(-u) du.
\]
By \cref{lemma:phi_integral}:
\[
\int_{-\infty}^{\infty} \Phi(u) \Phi(-u) du = \frac{1}{\sqrt{\pi}},
\]

\textbf{Step 8: Final Expression for \( \mathbb{E}[D_{\mathrm{ReLU}}] \)}

Substituting the approximation into the expectation:
\[
\mathbb{E}[D_{\mathrm{ReLU}}] \approx \frac{\sigma^2}{\sqrt{2}\pi k}.
\]

\subsection{GELU Activation Function}

Consider the Gaussian Error Linear Unit (GELU) activation function defined as:
\[
f(z) = z \cdot \Phi\left(z \sqrt{\frac{2}{\pi}}\right),
\]
where \(\Phi(\cdot)\) is the cumulative distribution function (CDF) of the standard normal distribution. We aim to compute the expected M-loss:
\[
\mathbb{E}[D_{\mathrm{GELU}}] = \mathbb{E}_{x, a, b} \left[ \left| f\left( \frac{a + b}{2} \right) - \frac{f(a) + f(b)}{2} \right| \right],
\]
where:
\begin{itemize}
    \item \( x \sim \text{Uniform}(-k, k) \): Approximates the embedding through the pretrained model.
    \item \( a, b \sim \mathcal{N}(x, \sigma^2) \): Approximate embeddings through fine-tuned models, centered around \( x \).
    \item Assume \( k \gg \sigma \), as fine-tuned models have small shifts from the pre-trained model.
\end{itemize}

\textbf{Step 1: Small Perturbation Expansion (Taylor Approximation)}

Since \(\sigma \ll k\), the perturbations \(\delta_a, \delta_b \sim \mathcal{N}(0, \sigma^2)\) are small. We expand \(f(a)\) and \(f(b)\) around \(x\) using a second-order Taylor series:
\[
\begin{aligned}
f(a) &\approx f(x) + f'(x)\delta_a + \frac{1}{2}f''(x)\delta_a^2, \\
f(b) &\approx f(x) + f'(x)\delta_b + \frac{1}{2}f''(x)\delta_b^2.
\end{aligned}
\]
Similarly, the midpoint \(s = x + \frac{\delta_a + \delta_b}{2}\) is expanded as:
\[
f(s) \approx f(x) + f'(x)\left(\frac{\delta_a + \delta_b}{2}\right) + \frac{1}{2}f''(x)\left(\frac{\delta_a + \delta_b}{2}\right)^2.
\]

\textbf{Step 2: Approximate Expression for M-loss}

Substituting these expansions into \(D = |f(s) - \frac{f(a) + f(b)}{2}|\) and retaining terms up to second order, we obtain:
\[
\begin{aligned}
D &\approx \left| f'(x) \cdot \frac{\delta_a + \delta_b}{2} + \frac{f''(x)}{2} \cdot \frac{(\delta_a + \delta_b)^2}{4} - \frac{f'(x)(\delta_a + \delta_b) + \frac{f''(x)}{2}(\delta_a^2 + \delta_b^2)}{2} \right| \\
&= \left| \frac{f''(x)}{8} \left[ (\delta_a + \delta_b)^2 - 2(\delta_a^2 + \delta_b^2) \right] \right| \\
&= \frac{|f''(x)|}{8} \left| -\delta_a^2 + 2\delta_a\delta_b - \delta_b^2 \right| \\
&= \frac{|f''(x)|}{8} (\delta_a - \delta_b)^2.
\end{aligned}
\]

\textbf{Step 3: Expectation Calculation}

Using the independence of \(\delta_a\) and \(\delta_b\) (covariance is zero), the conditional expectation is:
\[
\mathbb{E}[D | x] \approx \frac{|f''(x)|}{8} \mathbb{E}\left[ (\delta_a - \delta_b)^2 \right] = \frac{|f''(x)|}{8} \cdot 2\sigma^2 = \frac{\sigma^2 |f''(x)|}{4}.
\]

\textbf{Step 4: Global Expectation Integration}

Since \(x \sim \text{Uniform}(-k, k)\), the total expectation is:
\[
\mathbb{E}[D_{\mathrm{GELU}}] \approx \frac{\sigma^2}{4k} \int_{-k}^{k} |f''(x)| dx.
\]
Under the assumption \(k \gg \sigma \), and $|f''(x)|$ vanishes with large absolute value of $x$, the integral limits can be extended to infinity:
\[
\int_{-k}^{k} |f''(x)| dx \approx \int_{-\infty}^{\infty} |f''(x)| dx.
\]

\textbf{Step 5: Second Derivative of GELU}

The second derivative of GELU, \(f''(x)\), is given by:
\[
f''(x) = \sqrt{\frac{2}{\pi}} \phi\left(x \sqrt{\frac{2}{\pi}}\right) \left(2 + x^2 \cdot \frac{2}{\pi}\right),
\]
where \(\phi(u) = \frac{1}{\sqrt{2\pi}} e^{-u^2/2}\) is the standard normal PDF. Since \(f''(x)\) is even and non-negative, the integral simplifies to:
\[
\int_{-\infty}^{\infty} |f''(x)| dx = 2 \int_{0}^{\infty} f''(x) dx = 2 \left[ f'(x) \right]_0^{\infty}.
\]
Noting that \(f'(x) \to 1\) as \(x \to \infty\) and \(f'(0) = \Phi(0) = \frac{1}{2}\), we have:
\[
\int_{-\infty}^{\infty} |f''(x)| dx = 2 \left(1 - \frac{1}{2}\right) = 1.
\]

\textbf{Step 6: Final Result}

Substituting the integral result, we obtain:

\[
\mathbb{E}[D_{\mathrm{GELU}}] \approx \frac{\sigma^2}{4k} \cdot 1 = \frac{\sigma^2}{4k}.
\]

\subsection{Leaky ReLU Activation Function}

Consider the Leaky ReLU activation function defined as:
\[
f(z) = \begin{cases} 
z & \text{if } z \geq 0, \\
\alpha z & \text{if } z < 0, 
\end{cases}
\]
where \( \alpha \) is a small constant (e.g., 0.01). We aim to compute the expected M-loss:
\[
\mathbb{E}[D_{\mathrm{LeakyReLU}}] = \mathbb{E}_{x, a, b} \left[ \left| f\left( \frac{a + b}{2} \right) - \frac{f(a) + f(b)}{2} \right| \right],
\]
where:
\begin{itemize}
    \item \( x \sim \text{Uniform}(-k, k) \): Approximates the embedding through the pretrained model.
    \item \( a, b \sim \mathcal{N}(x, \sigma^2) \): Approximate embeddings through fine-tuned models, centered around \( x \).
    \item Assume $k \gg \sigma$, as fine-tuned models have small amount shift from pre-trained model.
\end{itemize}

\textbf{Step 1: Variable Transformation}

Define:
\[
\delta_a = a - x \sim \mathcal{N}(0, \sigma^2), \quad \delta_b = b - x \sim \mathcal{N}(0, \sigma^2).
\]
Thus,
\[
s = \frac{a + b}{2} = x + \frac{\delta_a + \delta_b}{2}.
\]

\textbf{Step 2: Expression for M-loss (\( D \))}

Under the Leaky ReLU activation, the M-loss becomes:
\[
D = \left| f(s) - \frac{f(a) + f(b)}{2} \right| = \left| 
\begin{cases} 
s & \text{if } s \geq 0, \\
\alpha s & \text{if } s < 0 
\end{cases} 
- \frac{f(a) + f(b)}{2} 
\right|.
\]

\textbf{Step 3: Case Analysis}

The discrepancy \( D \) arises when \( a \) and \( b \) are on opposite sides of zero. Specifically:
\begin{enumerate}
    \item \textbf{Case 1}: Both \( a \geq 0 \) and \( b \geq 0 \) \(\Rightarrow\) \( D = 0 \).
    \item \textbf{Case 2}: Both \( a < 0 \) and \( b < 0 \) \(\Rightarrow\) \( D = 0 \).
    \item \textbf{Case 3}: \( a \geq 0 \) and \( b < 0 \) (denoted as Case A) \(\Rightarrow\) \( D > 0 \).
    \item \textbf{Case 4}: \( a < 0 \) and \( b \geq 0 \) (symmetrical to Case A) \(\Rightarrow\) \( D > 0 \).
\end{enumerate}

Due to symmetry, Cases A and A' (Cases 3 and 4) are identical in their contributions to \( D \).

\textbf{Step 4: Probability of Non-Zero \( D \)}

The probability that \( D > 0 \) given \( x \) is:
\[
P(D > 0 | x) = 2 \times P(a \geq 0, b < 0 | x) = 2 \times \Phi\left(\frac{x}{\sigma}\right) \times \Phi\left(-\frac{x}{\sigma}\right),
\]
where \( \Phi(\cdot) \) is the cumulative distribution function (CDF) of the standard normal distribution.

\textbf{Step 5: Expectation of \( D \) Given \( x \)}

In Case A (\( a \geq 0 \) and \( b < 0 \)):

\[
D = \left| 
\begin{cases} 
s - \frac{a + b}{2} & \text{if } s \geq 0, \\
\alpha s - \frac{f(a) + f(b)}{2} & \text{if } s < 0 
\end{cases} 
\right|.
\]

Substituting \( a = x + \delta_a \) and \( b = x + \delta_b \):

\[
s = x + \frac{\delta_a + \delta_b}{2}.
\]

Thus,

\[
D = \left| 
\begin{cases} 
x + \frac{\delta_a + \delta_b}{2} - \frac{(x + \delta_a) + (x + \delta_b)}{2} & \text{if } s \geq 0, \\
\alpha \left(x + \frac{\delta_a + \delta_b}{2}\right) - \frac{(x + \delta_a) + \alpha (x + \delta_b)}{2} & \text{if } s < 0 
\end{cases} 
\right|.
\]

Simplifying each subcase:

\textbf{Subcase A1: \( s \geq 0 \)}

\[
D = \left| x + \frac{\delta_a + \delta_b}{2} - \frac{2x + \delta_a + \delta_b}{2} \right| = \left| x + \frac{\delta_a + \delta_b}{2} - x - \frac{\delta_a + \delta_b}{2} \right| = 0.
\]
Thus, \( D = 0 \) in this subcase.

\textbf{Subcase A2: \( s < 0 \)}

\[
D = \left| \alpha \left(x + \frac{\delta_a + \delta_b}{2}\right) - \frac{(x + \delta_a) + \alpha (x + \delta_b)}{2} \right|.
\]
Simplifying:
\[
D = \left| \alpha x + \frac{\alpha (\delta_a + \delta_b)}{2} - \frac{(1 + \alpha)x + \delta_a + \alpha \delta_b}{2} \right| = \left| \alpha x + \frac{\alpha \delta_a + \alpha \delta_b}{2} - \frac{(1 + \alpha)x + \delta_a + \alpha \delta_b}{2} \right|.
\]
\[
= \left| \left( \alpha x - \frac{(1 + \alpha)x}{2} \right) + \left( \frac{\alpha \delta_a + \alpha \delta_b - \delta_a - \alpha \delta_b}{2} \right) \right| = \left| \frac{2\alpha - (1 + \alpha)}{2} x + \frac{(\alpha - 1) \delta_a}{2} \right|.
\]
\[
= \left| \frac{\alpha - 1}{2} x + \frac{(\alpha - 1) \delta_a}{2} \right| = \left| \frac{1 - \alpha}{2} (x + \delta_a) \right|.
\]
(Since \( \alpha < 1 \), \( \frac{\alpha - 1}{2} = -\frac{1 - \alpha}{2} \), and the absolute value removes the negative sign.)
\[
D = \frac{1 - \alpha}{2} |x + \delta_a |.
\]

\textbf{Step 6: Expectation in Subcase A2}

As \( \sigma \ll k \), with the symmetric distribution of $x$, we approximate:
\[
\mathbb{E}[|x + \delta_a | | s <0, a \geq 0, b <0] \approx \mathbb{E}[|\delta_a|] = \sigma \sqrt{\frac{2}{\pi}}.
\]
Thus,
\[
\mathbb{E}[D | x, \text{Case A2}] = \frac{1 - \alpha}{2} \times \sigma \sqrt{\frac{2}{\pi}}.
\]

\textbf{Step 7: Total Expectation Over \( x \)}

Considering both Cases A and A' (due to symmetry):
\[
\mathbb{E}[D_{\mathrm{LeakyReLU}}] = \mathbb{E}_{x} \left[ \mathbb{E}[D | x] \right] = \mathbb{E}_{x} \left[ \frac{1 - \alpha}{2} \times \sigma \sqrt{\frac{2}{\pi}} \times 2 \Phi\left(\frac{x}{\sigma}\right) \Phi\left(-\frac{x}{\sigma}\right) \right].
\]
Simplifying:
\[
\mathbb{E}[D_{\mathrm{LeakyReLU}}] = \frac{1 - \alpha}{2} \times \sigma \sqrt{\frac{2}{\pi}} \times \frac{1}{2k} \int_{-k}^{k} 2 \Phi\left(\frac{x}{\sigma}\right) \Phi\left(-\frac{x}{\sigma}\right) dx.
\]
\[
= \frac{1 - \alpha}{2} \times \sigma \sqrt{\frac{2}{\pi}} \times \frac{1}{2k} \times 2 \times \frac{\sigma}{\sqrt{\pi}} = \frac{(1 - \alpha) \sigma^2}{\sqrt{2}\pi k}.
\]
(In the above expression we use \cref{lemma:phi_integral} to calculate the integral.)

\section{Hyperparameter Range for Experiment on ViT-B/32 Model} \label{hypp}

\begin{table}[htbp]
\centering
\caption{Hyperparameter searching range for different methods in the experiment of ViT-B/32 model. The best settings are in bold.}
\label{tab:hyperparameters}
\begin{tabular}{ll}
\\
\toprule
\textbf{Method} & \textbf{Hyperparameters} \\
\midrule
TIES & $k = \textbf{0.2}, 0.3, 0.4, 0.5$ \\
M-TIES & $k = \textbf{0.2}, 0.3, 0.4, 0.5$, fixed $e = \textbf{0.1}$ \\
Task Arithmetic & $\lambda = 0.5, 0.8, 1.0, 1.2, \textbf{1.5}, 2.0$ \\
DARE & $k = 0.3, 0.5, 0.7, \textbf{0.8}$ \\
\bottomrule
\end{tabular}
\end{table}

\begin{table}[htbp]
\centering
\caption{Fixed hyperparameter for different methods in the experiment of ViT-L/14 model.}
\label{tab:hyperparameters}
\begin{tabular}{ll}
\\
\toprule
\textbf{Method} & \textbf{Hyperparameters} \\
\midrule
TIES & $k =  0.4$ \\
M-TIES & $k = 0.4$, $e = 0.1$ \\
Task Arithmetic & $\lambda =  1.5$ \\
DARE & $k = 0.8$ \\
\bottomrule
\end{tabular}
\end{table}
\end{document}